\documentclass{article}

\usepackage{proceed2e}

\usepackage{balance}
\usepackage{mathrsfs}
\usepackage{amsthm} 
\usepackage{amsmath}
\usepackage{amssymb}
\usepackage{graphics} 
\usepackage{graphicx}
\usepackage{hyperref}
\usepackage{bm}
\usepackage{xcolor}
\usepackage[sort]{natbib}

\newcommand{\inspace}{\ensuremath{\mathcal{X}}}   
\newcommand{\pp}[1]{\ensuremath{\mathbb{#1}}} 
\newcommand{\pspace}{\ensuremath{\mathfrak{P}_{\inspace}}}    
\newcommand{\hbspace}{\ensuremath{\mathcal{H}}}    

\newcommand{\abbrvmm}[1]{\ensuremath{\mu_{#1}}}

\newcommand{\dd}{\, \mathrm{d}} 

\newtheorem{definition}{Definition}
\newtheorem{theorem}[definition]{Theorem}

\providecommand{\abs}[1]{\lvert#1\rvert}  
\providecommand{\norm}[1]{\lVert#1\rVert}

\bibpunct[; ]{(}{)}{,}{a}{}{;}

\usepackage{caption}
\usepackage{subcaption}

\title{One-Class Support Measure Machines for Group Anomaly Detection}

\author{ {\bf Krikamol Muandet} \\
  Empirical Inference Department \\
  Max Planck Institute for Intelligent Systems \\
  Spemannstra{\ss}e 38, 72076 T\"{u}bingen \\
  \And {\bf Bernhard Sch\"{o}lkopf} \\
  Empirical Inference Department \\
  Max Planck Institute for Intelligent Systems \\
  Spemannstra{\ss}e 38, 72076 T\"{u}bingen
}

\begin{document}
\maketitle

\begin{abstract}
  We propose one-class support measure machines (OCSMMs) for group anomaly detection. Unlike traditional anomaly detection, OCSMMs aim at recognizing anomalous aggregate behaviors of data points. The OCSMMs generalize well-known one-class support vector machines (OCSVMs) to a space of probability measures. By formulating the problem as quantile estimation on distributions, we can establish interesting connections to the OCSVMs and variable kernel density estimators (VKDEs) over the input space on which the distributions are defined, bridging the gap between large-margin methods and kernel density estimators. In particular, we show that various types of VKDEs can be considered as solutions to a class of regularization problems studied in this paper. Experiments on Sloan Digital Sky Survey dataset and High Energy Particle Physics dataset demonstrate the benefits of the proposed framework in real-world applications.
\end{abstract}

\section{Introduction}
\label{sec:introduction}

Anomaly detection is one of the most important tools in all data-driven scientific disciplines. Data that do not conform to the expected behaviors often bear some interesting characteristics and can help domain experts better understand the problem at hand. However, in the era of data explosion, the anomaly may appear not only in the data themselves, but also as a result of their interactions. The main objective of this paper is to investigate the latter type of anomalies. To be consistent with the previous works \citep{Poczos11:Divergence,Xiong11:HPM,Liang11:FGM}, we will refer to this problem as a group anomaly detection, as opposed to a traditional point anomaly detection.

Like traditional point anomaly detection, the group anomaly detection refers to a problem of finding patterns in groups of data that do not conform to expected behaviors \citep{Poczos11:Divergence,Xiong11:HPM,Liang11:FGM}. That is, an ultimate goal is to detect interesting aggregate behaviors of data points among several groups. In principle, anomalous groups may consist of individually anomalous points, which are relatively easy to detect. On the other hand, anomalous groups of relatively normal points, whose behavior as a group is unusual, is much more difficult to detect. In this work, we are interested in the latter type of group anomalies. Figure \ref{fig:group_anomaly} illustrates this scenario.

\begin{figure}[t!] 
  \centering
  \includegraphics[width=\linewidth]{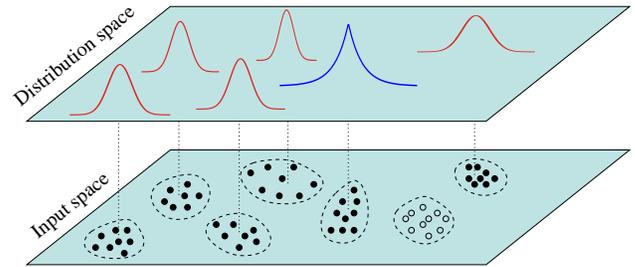}
  \caption{An illustration of two types of group anomalies. An anomalous group may be a group of anomalous samples which is easy to detect (unfilled points). In this paper, we are interested in detecting anomalous groups of normal samples (filled points) which is more difficult to detect because of the higher-order statistics. Note that group anomaly we are interested in can only be observed in the space of distributions.}
  \label{fig:group_anomaly}
\end{figure}

Group anomaly detection may shed light in a wide range of applications. For example, a Sloan Digital Sky Survey (SDSS) has produced a tremendous amount of astronomical data. It is therefore very crucial to detect rare objects such as stars, galaxies, or quasars that might lead to a scientific discovery. In addition to individual celestial objects, investigating groups of them may help astronomers understand the universe on larger scales. For instance, the anomalous group of galaxies, which is the smallest aggregates of galaxies, may reveal interesting phenomena, e.g., the gravitational interactions of galaxies.

Likewise, a new physical phenomena in high energy particle physics such as Higgs boson appear as a tiny excesses of certain types of collision events among a vast background of known physics in particle detectors \citep{Bhat10:PP,Vatanan12:SSCA}. Investigating each collision event individually is no longer sufficient as the individual events may not be anomalies by themselves, but their occurrence together as a group is anomalous. Hence, we need a powerful algorithm to detect such a rare and highly structured anomaly.

Lastly, the algorithm proposed in this paper can be applied to point anomaly detection with substantial and heterogeneous uncertainties. For example, it is often costly and time-consuming to obtain the full spectra of astronomical objects. Instead, relatively noisier measurements are usually made. In addition, the estimated uncertainty which represents the uncertainty one would obtain from multiple observations is also available. Incorporating these uncertainties has been shown to improve the performance of the learning systems \citep{Kirkpatrick11:Likelihood,Bovy11:XDQSO,Ross12:DR9QTS}.

The anomaly detection has been intensively studied (\citet{Chandola09:AnomalySurvey} and references therein). However, few attempts have been made on developing successful group anomaly detection algorithms. For example, a straightforward approach is to define a set of features for each group and apply standard point anomaly detection \citep{Chan05:TimeSeries}. Despite its simplicity, this approach requires a specific domain knowledge to construct appropriate sets of features. Another possibility is to first identify the individually anomalous points and then find their aggregations \citep{Das08:Anomaly}. Again, this approach relies only on the detection of anomalous points and thus cannot find the anomalous groups in which their members are perfectly normal. Successful group anomaly detectors should be able to incorporate the higher-order statistics of the groups.

Recently, a family of hierarchical probabilistic models based on a Latent Dirichlet Allocation (LDA) \citep{Blei03:LDA} has been proposed to cope with both types of group anomalies \citep{Xiong11:HPM,Liang11:FGM}. In these models, the data points in each group are assumed to be one of the $K$ different types and generated by a mixture of $K$ Gaussian distributions. Although the distributions over these $K$ types can vary across $M$ groups, they share common generator. The groups that have small probabilities under the model are marked as anomalies using scoring criteria defined as a combination of a point-based anomaly score and a group-based anomaly score. The Flexible Genre Model (FGM) recently extends this idea to model more complex group structures \citep{Liang11:FGM}.

Instead of employing a generative approach, we propose a simple and efficient discriminative way of detecting group anomaly. In this work, $M$ groups of data points are represented by a set of $M$ probability distributions assumed to be i.i.d. realization of some unknown distribution $\mathscr{P}$. In practice, only i.i.d samples from these distributions are observed. Hence, we can treat group anomaly detection as detecting the anomalous distributions based on their empirical samples. To allow for a practical algorithm, the distributions are mapped into the reproducing kernel Hilbert space (RKHS) using the kernel mean embedding. By working directly with the distributions, the higher-order information arising from the aggregate behaviors of the data points can be incorporated efficiently.


\section{Quantile Estimation on Probability Distributions} 
\label{sec:setting}

Let $\inspace$ denote a non-empty input space with associated $\sigma$-algebra $\mathcal{A}$, $\pp{P}$ denote the probability distribution on $(\inspace,\mathcal{A})$, and $\mathfrak{P}_{\inspace}$ denote the set of all probability distributions on $(\inspace,\mathcal{A})$. The space $\mathfrak{P}_{\inspace}$ is endowed with the topology of weak convergence and the associated Borel $\sigma$-algebra. 


We assume that there exists a distribution $\mathscr{P}$ on $\mathfrak{P}_{\inspace}$, where $\pp{P}_1,\dotsc,\pp{P}_{\ell}$ are i.i.d. realizations from $\mathscr{P}$, and the sample $S_i$ is made of $n_i$ i.i.d. samples distributed according to the distribution $\pp{P}_i$. In this work, we observe $\ell$ samples $S_i=\{x^{(i)}_k\}_{1\leq k\leq n_i}$ for $i=1,\dotsc,\ell$. For each sample $S_i$, $\widehat{\pp{P}}_i=\frac{1}{n_i}\sum_{j=1}^{n_i}\delta_{x_j^{(i)}}$ is the associated empirical distribution of $\pp{P}_i$.

In this work, we formulate a group anomaly detection problem as learning quantile function $q:\pspace\rightarrow\mathbb{R}$ to estimate the support of $\mathscr{P}$. Let $\mathcal{C}$ be a class of measurable subsets of $\pspace$ and $\lambda$ be a real-valued function defined on $\mathcal{C}$, the quantile function w.r.t. $(\mathscr{P},\mathcal{C},\lambda)$ is
\begin{equation*}
  q(\beta) = \inf\{\lambda(C): \mathscr{P}(C)\geq\beta, C\in\mathcal{C}\} \enspace, 
\end{equation*}
\noindent where $0 < \beta \leq 1$. In this paper, we consider when $\lambda$ is Lebesgue measure, in which case $C(\beta)$ is the minimum volume $C\in\mathcal{C}$ that contains at least a fraction $\beta$ of the probability mass of $\mathscr{P}$. Thus, the function $q$ can be used to test if any test distribution $\pp{P}_t$ is anomalous w.r.t. the training distributions.

Rather than estimating $C(\beta)$ in the space of distributions directly, we first map the distributions into a feature space via a positive semi-definite kernel $k$. Our class $\mathcal{C}$ is then implicitly defined as the set of half-spaces in the feature space. Specifically, $C_{\mathbf{w}}=\{\pp{P} \; | \; f_{\mathbf{w}}(\pp{P}) \geq \rho\}$ where $(\mathbf{w},\rho)$ are respectively a weight vector and an offset parametrizing a hyperplane in the feature space associated with the kernel $k$. The optimal $(\mathbf{w},\rho)$ is obtained by minimizing a regularizer which controls the smoothness of the estimated function describing $C$.

\section{One-Class Support Measure Machines}
\label{sec:ocsmm}


In order to work with the probability distributions efficiently, we represent the distributions as mean functions in a reproducing kernel Hilbert space (RKHS) \citep{Berlinet04:RKHS, Smola07Hilbert}. Formally, let $\hbspace$ denote an RKHS of functions $f:\inspace\rightarrow\mathbb{R}$ with reproducing kernel $k:\inspace\times\inspace\rightarrow\mathbb{R}$. The kernel mean map from $\pspace$ into $\hbspace$ is defined as
\begin{equation}
\label{eq:meanmap}
\mu : \pspace \rightarrow \hbspace, \enspace \mathbb{P} \longmapsto \int_{\inspace}k(x,\cdot)\dd
\mathbb{P}(x) \enspace .
\end{equation}
We assume that $k(x,\cdot)$ is bounded for any $x\in\inspace$. For any $\mathbb{P}$, letting $\abbrvmm{\pp{P}}=\mu(\pp{P})$, one can show that $\mathbb{E}_{\mathbb{P}}[f]=\langle\abbrvmm{\mathbb{P}},f \rangle_{\hbspace}$, for all $f\in\hbspace$. 

The following theorem due to \citet{Fukumizu04:RKHS} and \citet{Sriperumbudur10:Metrics} gives a promising property of representing distributions as mean elements in the RKHS. 

\begin{theorem}
  \label{thm:characteristic}
  The kernel $k$ is characteristic if and only if the map \eqref{eq:meanmap} is injective.
\end{theorem}

Examples of characteristic kernels include Gaussian RBF kernel and Laplace kernel. Using the characteristic kernel $k$, Theorem \ref{thm:characteristic} implies that the map \eqref{eq:meanmap} preserves all information about the distributions. Hence, one can apply many existing kernel-based learning algorithms to the distributions as if they are individual samples with no information loss.

Intuitively, one may view the mean embeddings of the distributions as their feature representations. Thus, our approach is in line with previous attempts in group anomaly detection that find a set of appropriate features for each group. On the one hand, however, the mean embedding approach captures all necessary information about the groups without relying heavily on a specific domain knowledge. On the other hand, it is flexible to choose the feature representation that is suitable to the problem at hand via the choice of the kernel $k$.

\subsection{OCSMM Formulation}

Using the mean embedding representation \eqref{eq:meanmap}, the primal optimization problem for one-class SMM can be subsequently formulated in an analogous way to the one-class SVM \citep{scholkopf01:OCSVM} as follow:
\begin{subequations} 
  \label{eq:ocsmm-primal}
  \begin{align}
    & \underset{\textbf{w},b,\xi,\rho}{\text{minimize}} & & 
    \frac{1}{2} \langle \textbf{w},\textbf{w} \rangle_{\hbspace} - \rho + \frac{1}{\nu\ell}\sum_{i=1}^{\ell}\xi_i \\
    & \text{subject to} & & \langle \textbf{w},\abbrvmm{\pp{P}_i}\rangle_{\hbspace} \geq \rho - \xi_i, \xi_i\geq 0
  \end{align}
\end{subequations}
\noindent where $\xi_i$ denote slack variables and $\nu\in(0,1]$ is a trade-off parameter corresponding to an expected fraction of outliers within the feature space. The trade-off $\nu$ is an upper bound on the fraction of outliers and lower bound on the fraction of support measures \citep{scholkopf01:OCSVM}. 

The trade-off parameter $\nu$ plays an important role in group anomaly detection. Small $\nu$ implies that anomalous groups are rare compared to the normal groups. Too small $\nu$ leads to some anomalous groups being rejected. On the other hand, large $\nu$ implies that anomalous groups are common. Too large $\nu$ leads to some normal groups being accepted as anomaly. As group anomaly is subtle, one need to choose $\nu$ very carefully to reduce the effort in the interpretation of the results.

By introducing Lagrange multipliers $\bm{\alpha}$, we have $\mathbf{w} = \sum_{i=1}^{\ell}\alpha_i\mu_{\pp{P}_i} = \sum_{i=1}^{\ell}\alpha_i\mathbb{E}_{\pp{P}_i}[k(x,\cdot)]$ and the dual form of \eqref{eq:ocsmm-primal} can be written as
\begin{subequations}
  \label{eq:ocsmm-dual} 
  \begin{align}
    & \underset{\bm{\alpha}}{\text{minimize}} & & 
    \frac{1}{2} \sum_{i=1}^{\ell}\sum_{j=1}^{\ell}\alpha_i\alpha_j\langle\mu_{\pp{P}_i},\mu_{\pp{P}_j}\rangle_{\hbspace} \\
    & \text{subject to} & & 0\leq\alpha_i\leq\frac{1}{\nu\ell},\; \sum_{i=1}^{\ell}\alpha_i = 1 \,.
  \end{align}
\end{subequations}
 
Note that the dual form is a quadratic programming and depends on the inner product $\langle\mu_{\pp{P}_i},\mu_{\pp{P}_j}\rangle_{\hbspace}$. Given that we can compute $\langle\mu_{\pp{P}_i},\mu_{\pp{P}_j}\rangle_{\hbspace}$, we can employ the standard QP solvers to solve \eqref{eq:ocsmm-dual}. 

\subsection{Kernels on Probability Distributions}
\label{sec:kernel-dist}

From \eqref{eq:ocsmm-dual}, we can see that $\abbrvmm{\pp{P}}$ is a feature map associated with the kernel $K:\pspace\times\pspace\rightarrow\mathbb{R}$, defined as $K(\mathbb{P}_i,\mathbb{P}_j)=\langle\abbrvmm{\pp{P}_i},\abbrvmm{\pp{P}_j}\rangle_{\hbspace}$. It follows from Fubini's theorem and reproducing property of $\hbspace$ that 
\begin{eqnarray}
\langle \mu_{\pp{P}_i},\mu_{\pp{P}_j}\rangle_{\hbspace} &=& \iint\langle k(x,\cdot),k(y,\cdot)\rangle_{\hbspace} \dd\mathbb{P}_i(x)\dd\mathbb{P}_j(y) \nonumber \\
&=& \iint k(x,y)\dd\mathbb{P}_i(x)\dd\mathbb{P}_j(y) \enspace . \label{eq:expected-kernel}
\end{eqnarray}
Hence, $K$ is a positive definite kernel on $\pspace$. Given the sample sets $S_1,\dotsc,S_{\ell}$, one can estimate \eqref{eq:expected-kernel} by
\begin{equation}
  \label{eq:empirical-kernel}
  K(\widehat{\pp{P}}_i,\widehat{\pp{P}}_j) 
  = \dfrac{1}{n_i\cdot n_j}\sum_{k=1}^{n_i}\sum_{l=1}^{n_j} k(x^{(i)}_k,x^{(j)}_l)
\end{equation}
\noindent where $x^{(i)}_k\in S_i$, $x^{(j)}_l\in S_j$, and $n_i$ is the number of samples in $S_i$ for $i=1,\dotsc,\ell$.

Previous works in kernel-based anomaly detection have shown that the Gaussian RBF kernel is more suitable than some other kernels such as polynomial kernels \citep{Hoffmann07:KPCA}. Thus we will focus primarily on the Gaussian RBF kernel given by 
\begin{equation}
  \label{eq:rbf-kernel}
  k_{\sigma}(x,x') = \exp\left(-\frac{\norm{x-x'}^2}{2\sigma^2}\right), \enspace x,x'\in\inspace
\end{equation}
\noindent where $\sigma>0$ is a bandwidth parameter. In the sequel, we denote the reproducing kernel Hilbert space associated with kernel $k_{\sigma}$ by $\hbspace_{\sigma}$. Also, let $\Phi:\inspace\rightarrow\hbspace_{\sigma}$ be a feature map such that $k_{\sigma}(x,x')=\langle\Phi(x),\Phi(x')\rangle_{\hbspace_{\sigma}}$.

In group anomaly detection, we always observe the i.i.d. samples from the distribution underlying the group. Thus, it is natural to use the empirical kernel \eqref{eq:empirical-kernel}. However, one may relax this assumption and apply the kernel \eqref{eq:expected-kernel} directly. For instance, if we have a Gaussian distribution $\pp{P}_i=\mathcal{N}(m_i,\Sigma_i)$ and a Gaussian RBF kernel $k_{\sigma}$, we can compute the kernel analytically by
\begin{equation}
  \label{eq:analytic-kernel}
  K(\pp{P}_i,\pp{P}_j) = \frac{\exp\left(-\frac{1}{2}(m_i-m_j)^{\mathsf{T}}B^{-1}(m_i-m_j)\right)}{\abs{\frac{1}{\sigma^2}\Sigma_i + \frac{1}{\sigma^2}\Sigma_j + \mathbf{I}}^{\frac{1}{2}}}
\end{equation}
\noindent where $B = \Sigma_i+\Sigma_j + \sigma^2\mathbf{I}$. This kernel is particularly useful when one want to incorporate the point-wise uncertainty of the observation into the learning algorithm \citep{Muandet12:SMM}. More details will be given in Section \ref{sec:connection} and \ref{sec:experiments}).

\section{Theoretical Analysis}
\label{sec:analysis}

This section presents some theoretical analyses. The geometrical interpretation of OCSMMs is given in Section \ref{sec:geometric}. Then, we discuss the connection of OCSMM to the kernel density estimator in Section \ref{sec:connection}. In the sequel, we will focus on the translation-invariant kernel function to simplify the analysis.

\subsection{Geometric Interpretation}
\label{sec:geometric}

\begin{figure*}[t]
  \centering
  \begin{subfigure}[b]{0.32\linewidth}
    \centering
    \includegraphics[width=1.5in]{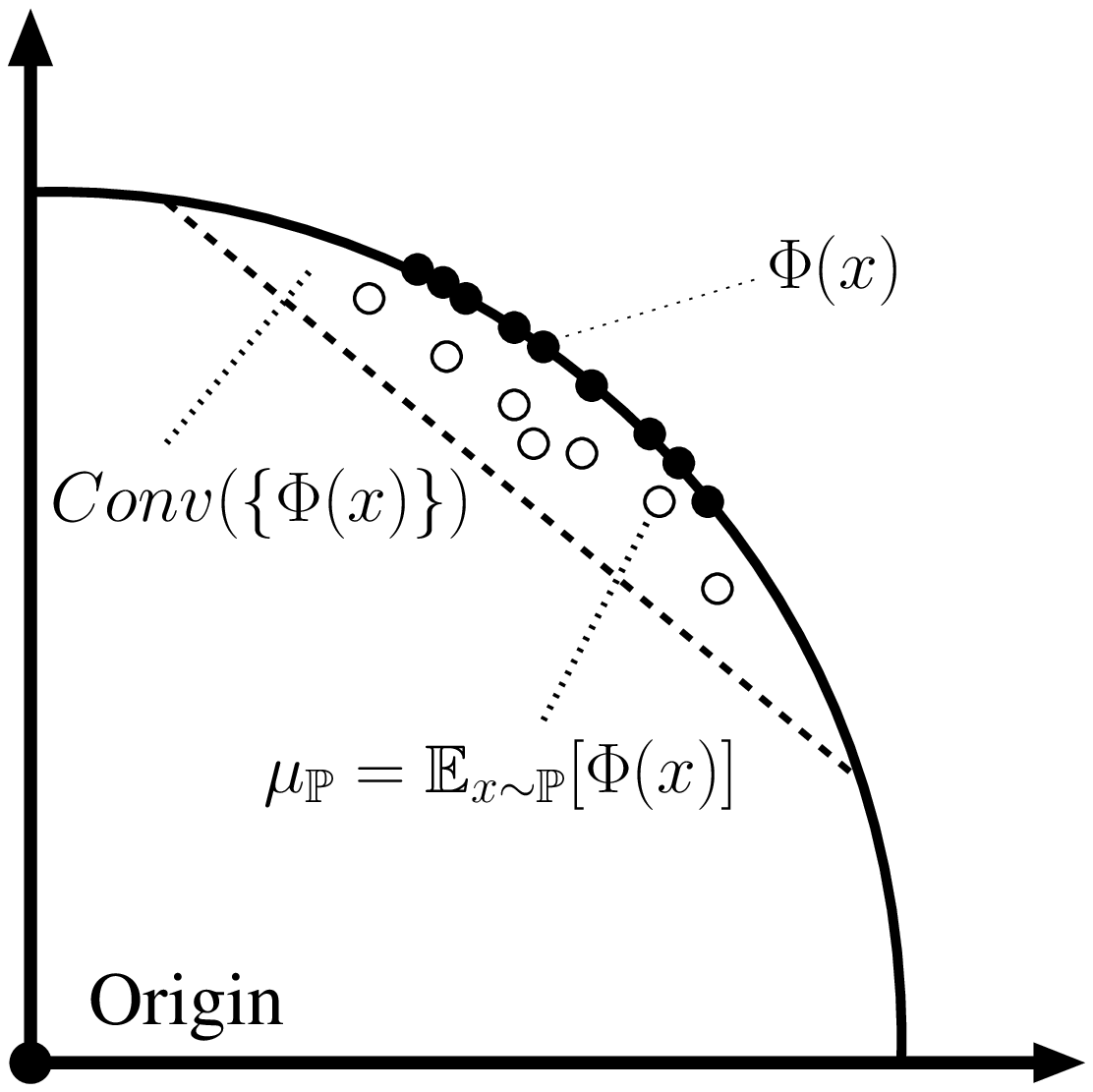}
    \caption{feature map and mean map}
    \label{fig:meanmaps}
  \end{subfigure}  
  \begin{subfigure}[b]{0.32\linewidth}
    \centering
    \includegraphics[width=1.5in]{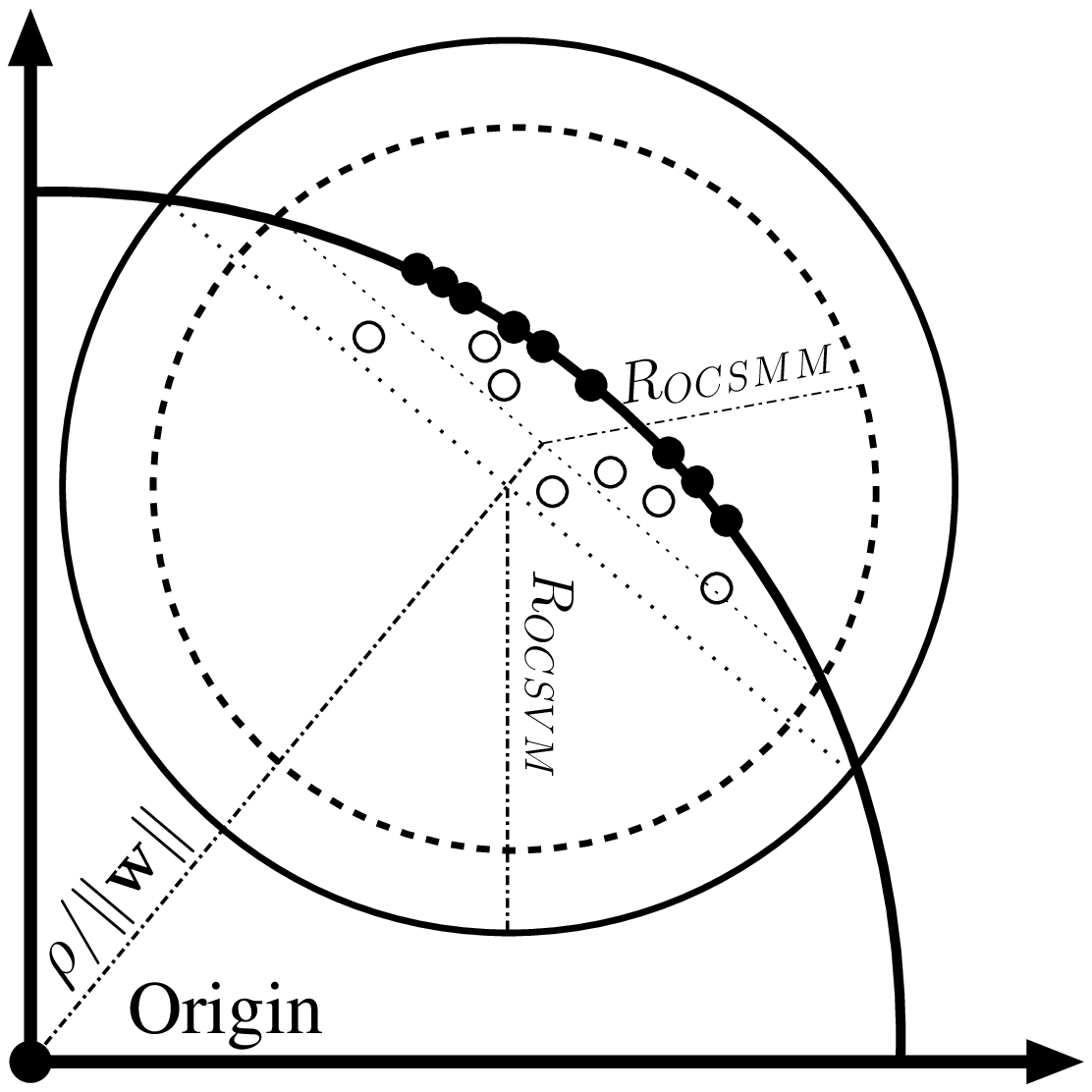}
    \caption{minimum enclosing sphere}
    \label{fig:sphere}
  \end{subfigure}
  \begin{subfigure}[b]{0.32\linewidth}
    \centering
    \includegraphics[width=1.5in]{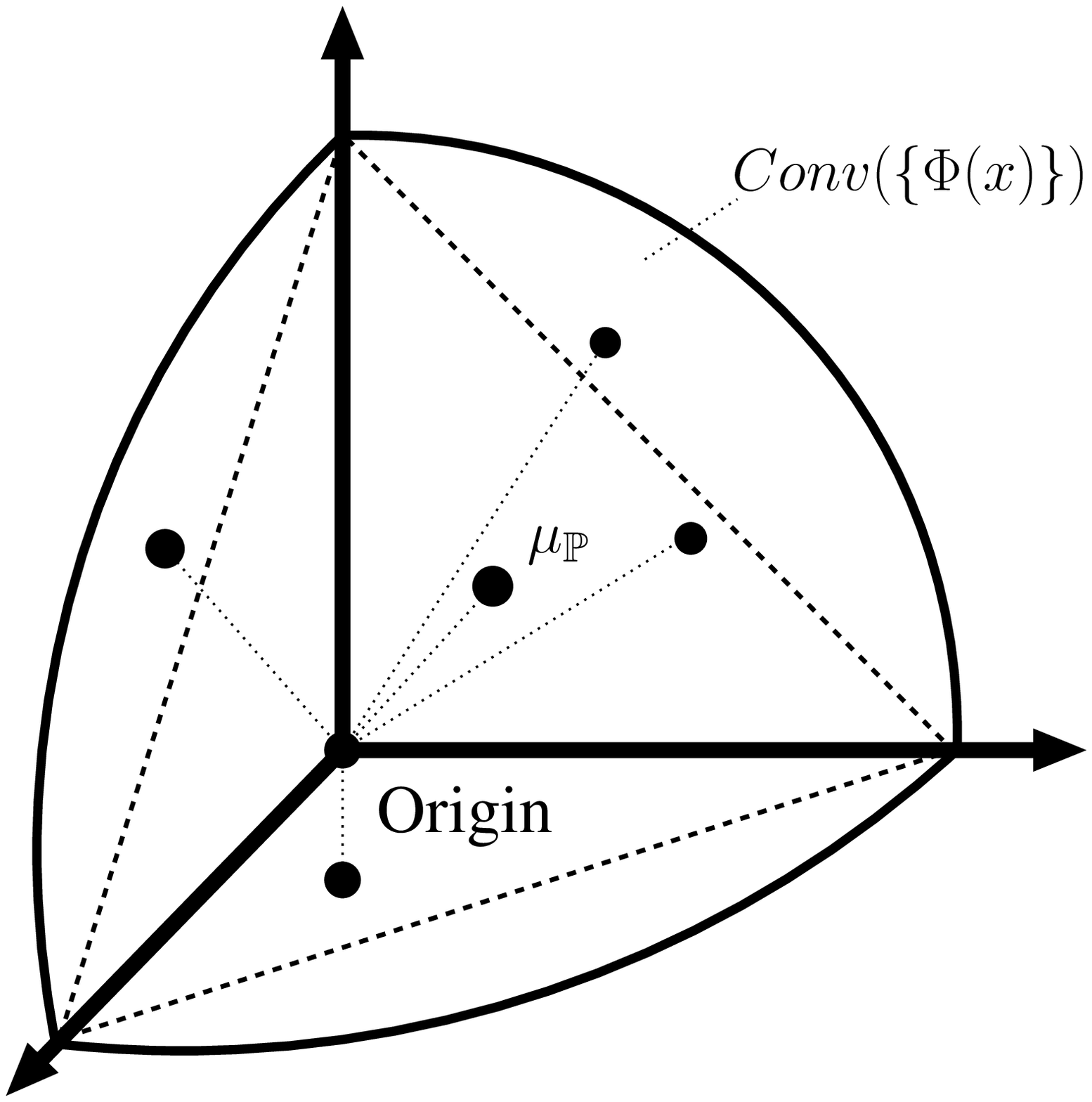}
    \caption{spherical normalization} 
    \label{fig:3dball-rkhs}
  \end{subfigure}
  \caption{(\subref{fig:meanmaps}) The two dimensional representation of the RKHS of Gaussian RBF kernels. Since the kernels depend only on $x-x'$, $k(x,x)$ is constant. Therefore, all feature maps $\Phi(x)$ (black dots) lie on a sphere in feature space. Hence, for any probability distribution $\pp{P}$, its mean embedding $\abbrvmm{\pp{P}}$ always lies in the convex hull of the feature maps, which in this case, forms a segment of the sphere. (\subref{fig:sphere}) In general, the solution of OCSMM is different from the minimum enclosing sphere. (\subref{fig:3dball-rkhs}) Three dimensional sphere in the feature space. For the Gaussian RBF kernel, the kernel mean embeddings of all distributions always lie inside the segment of the sphere. In addition, the angle between any pair of mean embeddings is always greater than zero. Consequently, the mean embeddings can be scaled, e.g., to lie on the sphere, and the map is still injective.}
  \label{fig:rbf-rkhs}
\end{figure*}

For translation-invariant kernel, $k(x,x)$ is constant for all $x\in\inspace$. That is, $\norm{\Phi(x)}_{\hbspace}=\tau$ for some constant $\rho$. This implies that all of the images $\Phi(x)$ lie on the sphere in the feature space (cf. Figure \ref{fig:meanmaps}). Consequently, the following inequality holds 
\begin{eqnarray*}
  \norm{\abbrvmm{\pp{P}}}_{\hbspace} = \left\| \int k(x,\cdot) \dd\pp{P}(x)\right\|_{\hbspace} \leq \int \norm{k(x,\cdot)}_{\hbspace}\dd\pp{P}(x) = \tau \, ,
\end{eqnarray*}
\noindent which shows that all mean embeddings lie inside the sphere (cf. Figure \ref{fig:meanmaps}). As a result, we can establish the existence and uniqueness of the separating hyperplane $\textbf{w}$ in \eqref{eq:ocsmm-primal} through the following theorem.

\begin{theorem}
  \label{thm:hyperplane}
  There exists a unique separating hyperplane $\textbf{w}$ as a solution to \eqref{eq:ocsmm-primal} that separates $\mu_{\pp{P}_1},\mu_{\pp{P}_2},\ldots,\mu_{\pp{P}_{\ell}}$ from the origin.
\end{theorem}

\begin{proof}
  Due to the separability of the feature maps $\Phi(x)$, the convex hull of the mean embeddings $\mu_{\pp{P}_1},\mu_{\pp{P}_2},\ldots,\mu_{\pp{P}_{\ell}}$ does not contain the origin. The existence and uniqueness of the hyperplane then follows from the supporting hyperplane theorem \citep{Scholkopf2001:LKS}.
\end{proof}

By Theorem \ref{thm:hyperplane}, the OCSMM is a simple generalization of OCSVM to the space of probability distributions. Furthermore, the straightforward generalization will allow for a direct application of an efficient learning algorithm as well as existing theoretical results.

There is a well-known connection between the solution of OCSVM with translation invariant kernels and the center of the minimum enclosing sphere (MES) \citep{Tax99:SVDD,Tax04:SVD}. Intuitively, this is not the case for OCSMM, even when the kernel $k$ is translation-invariant, as illustrated in Figure \ref{fig:sphere}. Fortunately, the connection between OCSMM and MES can be made precise by applying the spherical normalization
\begin{equation}
  \label{eq:normalized-kernel}
  \langle \abbrvmm{\pp{P}},\abbrvmm{\pp{Q}}\rangle_{\hbspace} 
  \longmapsto \frac{\langle \abbrvmm{\pp{P}},\abbrvmm{\pp{Q}}\rangle_{\hbspace} }
  {\sqrt{\langle \abbrvmm{\pp{P}},\abbrvmm{\pp{P}}\rangle_{\hbspace}\langle \abbrvmm{\pp{Q}},\abbrvmm{\pp{Q}}\rangle_{\hbspace}}}
\end{equation}

After the normalization, $\norm{\abbrvmm{\pp{P}}}_{\hbspace}=1$ for all $\pp{P}\in\pspace$. That is, all mean embeddings lie on the unit sphere in the feature space. Consequently, the OCSMM and MES are equivalent after the normalization. 

Given the equivalence between OCSMM and MES, it is natural to ask if the spherical normalization \eqref{eq:normalized-kernel} preserves the injectivity of the Hilbert space embedding. In other words, is there an information loss after the normalization? The following theorem answers this question for kernel $k$ that satisfies some reasonable assumptions.

\begin{theorem}
  \label{thm:preserve-injectivity}
  Assume that $k$ is characteristic and the samples are linearly independent in the feature space $\hbspace$. Then, the spherical normalization preserves the injectivity of the mapping $\mu: \pspace\rightarrow\hbspace$.
\end{theorem}

\begin{proof}
  Let us assume the normalization does not preserve the injectivity of the mapping. Thus, there exist two distinct probability distributions $\pp{P}$ and $\pp{Q}$ for which
  \begin{eqnarray*}
    \mu_{\pp{P}} &=& \mu_{\pp{Q}} \\
    \int k(x,\cdot)\dd\pp{P}(x) &=& \int k(x,\cdot)\dd\pp{Q}(x) \\
    \int k(x,\cdot)\dd(\pp{P}-\pp{Q})(x) &=& 0 \enspace .
  \end{eqnarray*}
  \noindent As $\pp{P}\neq\pp{Q}$, the last equality holds if and only if there exists $x\in\inspace$ for which $k(x,\cdot)$ are linearly dependent, which contradicts the assumption. Consequently, the spherical normalization must preserve the injectivity of the mapping.
\end{proof}

The Gaussian RBF kernel satisfies the assumption given in Theorem \ref{thm:preserve-injectivity} as the kernel matrix will be full-rank and thereby the samples are linearly independent in the feature space. Figure \ref{fig:3dball-rkhs} depicts an effect of the spherical normalization.

It is important to note that the spherical normalization does not necessarily improve the performance of the OCSMM. It ensures that all the information about the distributions are preserved.

\subsection{OCSMM and Density Estimation}
\label{sec:connection}

In this section we make a connection between the OCSMM and kernel density estimation (KDE). First, we give a definition of the KDE. Let $x_1,x_2,\ldots,x_n$ be an i.i.d. samples from some distribution $F$ with unknown density $f$, the KDE of $f$ is defined as
\begin{equation}
  \label{eq:kde}
  \hat{f}(y) = \frac{1}{nh}\sum_{i=1}^n k\left(\frac{y-x_i}{h}\right)
\end{equation}
For $\hat{f}$ to be a density, we require that the kernel satisfies $k(\cdot,\cdot)\geq 0$ and $\int k(x,\cdot) \dd x=1$, which includes, for example, the Gaussian kernel, the multivariate Student kernel, and the Laplacian kernel.

When $\nu=1$, it is well-known that, under some technical assumptions, the OCSVM corresponds exactly to the KDE \citep{scholkopf01:OCSVM}. That is, the solution $\mathbf{w}$ of \eqref{eq:ocsmm-primal} can be written as a uniform sum over training samples similar to \eqref{eq:kde}. Moreover, setting $\nu < 1$ yields a sparse representation where the summand consists of only support vectors of the OCSVM.

Interestingly, we can make a similar correspondence between the KDE and the OCSMM. It follows from Lemma 4 of \citet{Muandet12:SMM} that for certain classes of training probability distributions, the OCSMM on these distributions corresponds to the OCSVM on some training samples equipped with an appropriate kernel function. To understand this connection, consider the OCSMM with the Gaussian RBF kernel $k_{\sigma}$ and isotropic Gaussian distributions $\mathcal{N}(m_1;\sigma^2_1),\mathcal{N}(m_2;\sigma^2_2),\ldots,\mathcal{N}(m_n;\sigma^2_n)$\footnote{We adopt the Gaussian distributions here for the sake of simplicity. More general statement for non-Gaussian distributions follows straightforwardly.}. We analyze this scenario under two conditions:

\begin{paragraph}{(C1) Identical bandwidth.} If $\sigma_i=\sigma_j$ for all $1\leq i,j \leq n$, the OCSMM is equivalent to the OCSVM on the training samples $m_1,m_2,\ldots,m_n$ with Gaussian RBF kernel $k_{\sigma^2+\sigma^2_i}$ (cf. the kernel \eqref{eq:analytic-kernel}). Hence, the OCSMM corresponds to the OCSVM on the means of the distributions with kernel of larger bandwidth.
\end{paragraph}

\begin{paragraph}{(C2) Variable bandwidth.} Similarly, if $\sigma_i\neq\sigma_j$ for some $1\leq i,j \leq n$, the OCSMM is equivalent to the OCSVM on the training samples $m_1,m_2,\ldots,m_n$ with Gaussian RBF kernel $k_{\sigma^2+\sigma^2_i}$. Note that the kernel bandwidth may be different at each training samples. Thus, OCSMM in this case corresponds to the OCSVM with variable bandwidth parameters.
\end{paragraph}

On the one hand, the above scenario allows the OCSVM to cope with noisy/uncertain inputs, leading to more robust point anomaly detection algorithm. That is, we can treat the means as the measurements and the covariances as the measurement uncertainties (cf. Section \ref{sec:exp-noisy}). On the other hand, one can also interpret the OCSMM when $\nu=1$ as a generalization of traditional KDE, where we have a data-dependent bandwidth at each data point. This type of KDE is known in the statistics as variable kernel density estimators (VKDEs) \citep{Breiman77:VKDE,Abramson82:BV,Terrell1992:VKDE}. For $\nu < 1$, the OCSMM gives a sparse representation of the VKDE.

Formally, the VKDE is characterized by \eqref{eq:kde} with an adaptive bandwidth $h(x_i)$. For example, the bandwidth is adapted to be larger where the data are less dense, with the aim to reduce the bias. There are basically two different views of VKDE. The first is known as a \emph{balloon estimator} \citep{Terrell1992:VKDE}. Essentially, its bandwidth may depend only on the point at which the estimate is taken, i.e., the bandwidth in \eqref{eq:kde} may be written as $h(y)$. The second type of VKDE is a \emph{sample smoothing estimator} \citep{Terrell1992:VKDE}. As opposed to the balloon estimator, it is a mixture of individually scaled kernels centered at each observation, i.e., the bandwidth is $h(x_i)$. The advantage of balloon estimator is that it has a straightforward asymptotic analysis, but the final estimator may not be a density. The sample smoothing estimator is a density if $k$ is a density, but exhibits \emph{non-locality}. 

Both types of the VKDEs may be seen from the OCSMM point of view. Firstly, under the condition \textbf{(C1)}, the balloon estimator can be recovered by considering different test distribution $\pp{P}_t = \mathcal{N}(m_t;\sigma_t)$. As $\sigma_t\rightarrow 0$, one obtain the standard KDE on $m_t$. Similarly, the OCSMM under the condition \textbf{(C2)} with $\pp{P}_t=\delta_{m_t}$ gives the sample smoothing estimator. Interestingly, the OCSMM under the condition \textbf{(C2)} with $\pp{P}_t = \mathcal{N}(m_t;\sigma_t)$ results in a combination of these two types of the VKDEs.


In summary, we show that many variants of KDE can be seen as solutions to the regularization functional \eqref{eq:ocsmm-primal}, and thereby provides an insight into a connection between large-margin approach and kernel density estimation.

\section{Experiments}
\label{sec:experiments}

We firstly illustrate a fundamental difference between point and group anomaly detection problems. Then, we demonstrate an advantage of OCSMM on uncertain data when the noise is observed explicitly. Lastly, we compare the OCSMM with existing group anomaly detection techniques, namely, $K$-nearest neighbor (KNN) based anomaly detection \citep{Zhao09:Anomaly} with NP-L$_2$ divergence and NP-Renyi divergence \citep{Poczos11:Divergence}, and Multinomial Genre Model (MGM) \citep{Xiong11:HPM} on Sloan Digital Sky Survey (SDSS) dataset and High Energy Particle Physics dataset.

\begin{paragraph}{Model Selection and Setup.}
  One of the long-standing problems of one-class algorithms is model selection. Since no labeled data is available during training, we cannot perform cross validation. To encourage a fair comparison of different algorithms in our experiments, we will try out different parameter settings and report the best performance of each algorithm. We believe this simple approach should serve its purpose at reflecting the relative performance of different algorithms. We will employ the Gaussian RBF kernel \eqref{eq:rbf-kernel} throughout the experiments. For the OCSVM and the OCSMM, the bandwidth parameter $\sigma^2$ is fixed at $\mathrm{median}\{\norm{x^{(i)}_k - x^{(j)}_l}^2\}$ for all $i,j,k,l$ where $x^{(i)}_k$ denotes the $k$-th data point in the $i$-th group, and we consider $\nu=(0.1,0.2,\ldots,0.9)$. The OCSVM treats group means as training samples. For synthetic experiments with OCSMM, we use the empirical kernel \eqref{eq:empirical-kernel}, whereas the non-linear kernel $K(\pp{P}_i,\pp{P}_j)=\exp(\|\mu_{\pp{P}_i}-\mu_{\pp{P}_j}\|_{\hbspace}^2/2\gamma^2)$ will be used for real data where we set $\gamma=\sigma$. Our experiments suggest that these choices of parameters usually work well in practice. For KNN-L$_2$ and KNN-Renyi ($\alpha$=0.99), we consider when there are 3,5,7,9, and 11 nearest neighbors. For MGM, we follow the same experimental setup as in \citet{Xiong11:HPM}.
\end{paragraph}

\begin{figure*}[t!]
  \centering 
  \begin{subfigure}[b]{0.25\linewidth}
    \centering
    \includegraphics[width=\linewidth]{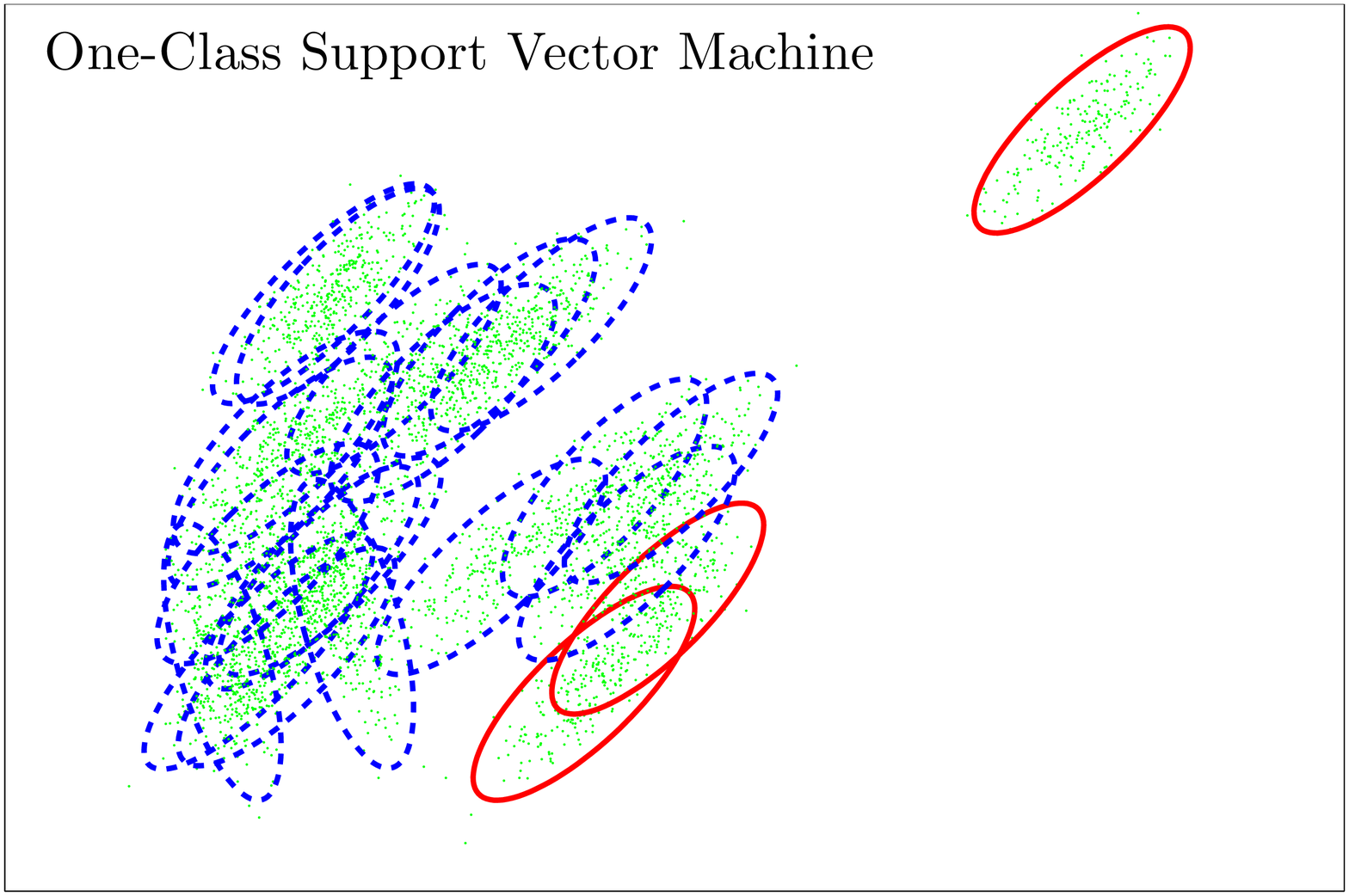} \\
    \includegraphics[width=\linewidth]{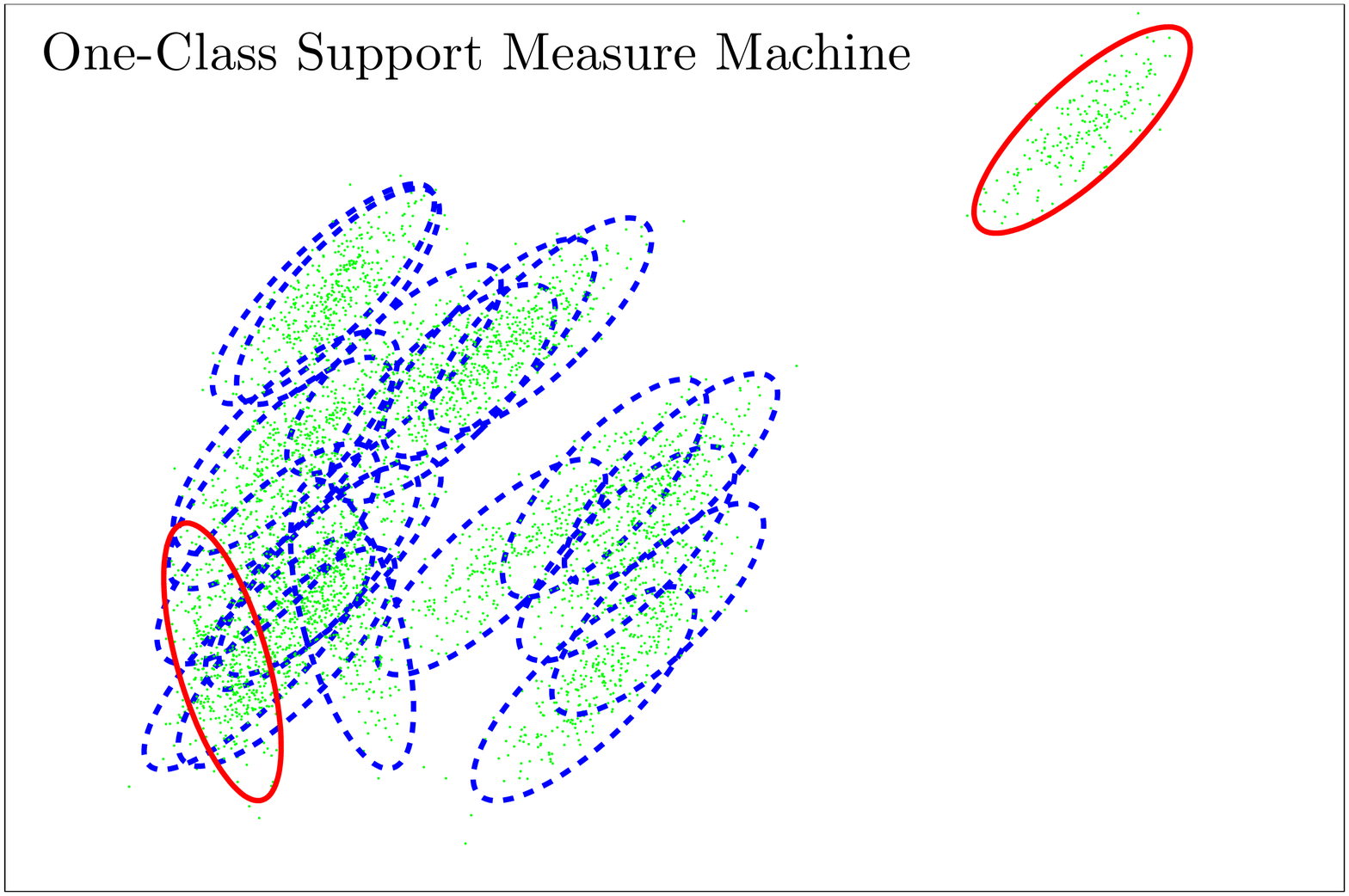}
    \caption{OCSVM vs OCSMM}
    \label{fig:ocsvm-ocsmm}
  \end{subfigure}
  \hfill 
  \begin{subfigure}[b]{0.7\linewidth} 
    \centering
    \includegraphics[width=0.98\linewidth]{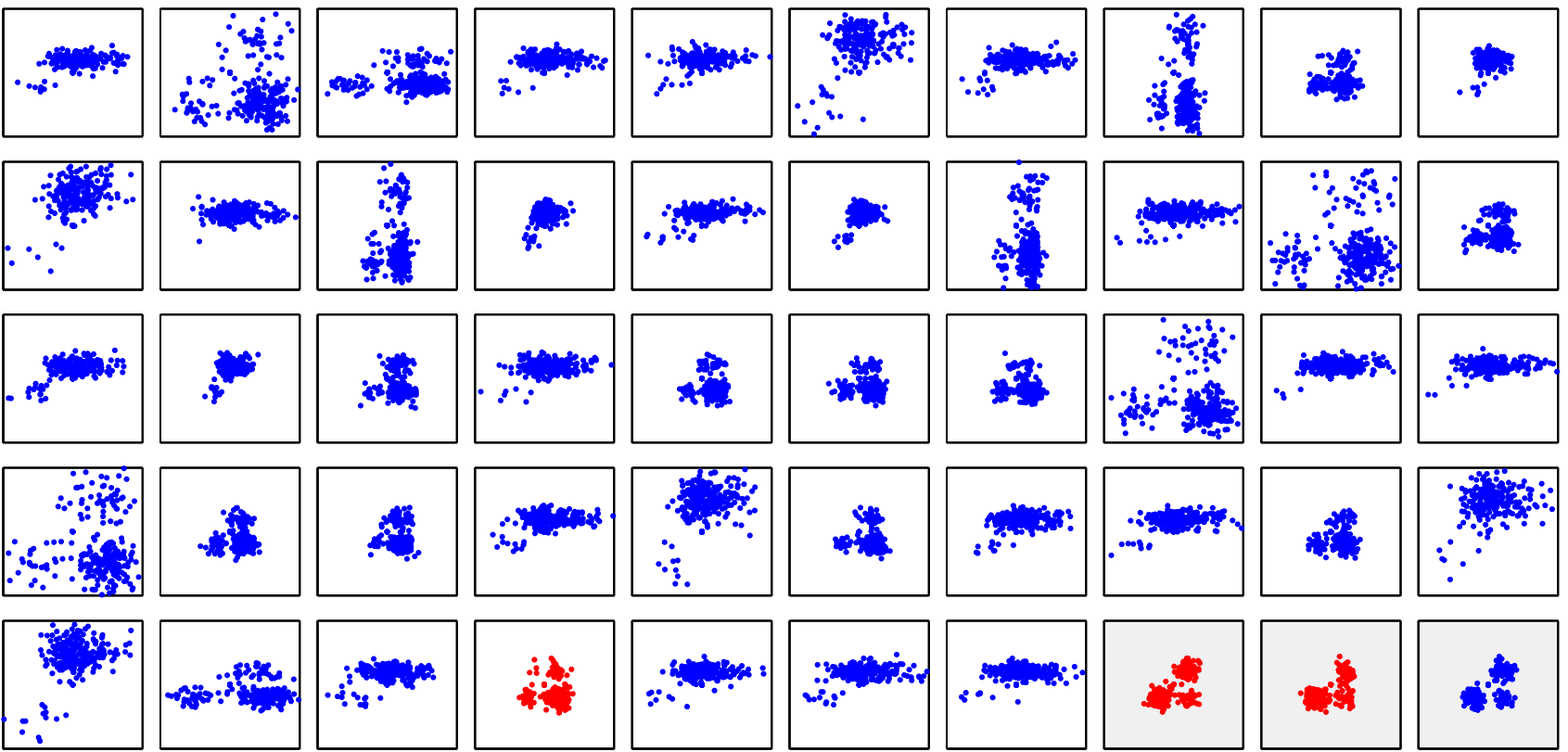}
    \caption{The results of the OCSMM on the mixture of Gaussian dataset}
    \label{fig:synthetic-results}
  \end{subfigure}
  \caption{(\subref{fig:ocsvm-ocsmm}) The results of group anomaly detection on synthetic data obtained from the OCSVM and the OCSMM. Blue dashed ovals represent the normal groups, whereas red ovals represent the detected anomalous groups. The OCSVM is only able to detect the anomalous groups that are spatially far from the rest in the dataset, whereas the OCSMM also takes into account other higher-order statistics and therefore can also detect anomalous groups which possess distinctive properties. (\subref{fig:synthetic-results}) The results of the OCSMM on the synthetic data of the mixture of Gaussian. The shaded boxes represent the anomalous groups that have different mixing proportion to the rest of the dataset. The OCSMM is able to detects the anomalous groups although they look reasonably normal and cannot be easily distinguished from other groups in the data set based only on an inspection.}
\end{figure*}



\subsection{Synthetic Data}


To illustrate the difference between point anomaly and group anomaly, we represent the group of data points by the 2-dimensional Gaussian distribution. We generate 20 normal groups with the covariance $\bm{\Sigma} = [0.01, 0.008; 0.008, 0.01]$. The means of these groups are drawn uniformly from $[0,1]$. Then, we generate 2 anomalous groups of Gaussian distributions whose covariances are rotated by 60 degree from the covariance $\bm{\Sigma}$. Furthermore, we perturb one of the normal groups to make it relatively far from the rest of the dataset to introduce an additional degree of anomaly (cf. Figure \ref{fig:ocsvm-ocsmm}). Lastly, we generate 100 samples from each of these distributions to form the training set.

For the OCSVM, we represent each group by its empirical average. Since the expected proportion of outliers in the dataset is approximately $10\%$, we use $\nu=0.1$ accordingly for both OCSVM and OCSMM. Figure \ref{fig:ocsvm-ocsmm} depicts the result which demonstrates that the OCSMM can detect anomalous aggregate patterns undetected by the OCSVM.

Then, we conduct similar experiment as that in \citet{Xiong11:HPM}. That is, the groups are represented as a mixture of four 2-dimensional Gaussian distributions. The means of the mixture components are $[-1,-1],[1,-1],[0,1],[1,1]$ and the covariances are all $\bm{\Sigma}=0.15\times\mathbf{I}_2$, where $\mathbf{I}_2$ denotes the 2D identity matrix. Then, we design two types of normal groups, which are specified by two mixing proportions $[0.22, 0.64, 0.03, 0.11]$ and $[0.22, 0.03, 0.64, 0.11]$, respectively. To generate a normal group, we first decide with probability $[0.48, 0.52]$ which mixing proportion will be used. Then, the data points are generated from mixture of Gaussian using the specified mixing proportion. The mixing proportion of the anomalous group is $[0.61, 0.1, 0.06, 0.23]$.

We generated $47$ normal groups with $n_i\sim \mathrm{Poisson}(300)$ instances in each group. Note that the individual samples in each group are perfectly normal compared to other samples. To test the performance of our technique, we inject the group anomalies, where the individual points are normal, but they together as a group look anomalous. In this anomalous group the individual points are samples from one of the $K=4$ normal topics, but the mixing proportion was different from both of the normal mixing proportions. We inject 3 anomalous groups into the data set. The OCSMM is trained using the same setting as in the previous experiment. The results are depicted in Figure \ref{fig:synthetic-results}.


\subsection{Noisy Data}
\label{sec:exp-noisy}

As discussed at the end of Section \ref{sec:kernel-dist}, the OCSMM may be adopted to learn from data points whose uncertainties are observed explicitly. To illustrate this claim, we generate samples from the unit circle using $x=\cos\theta+\varepsilon$ and $y=\sin\theta+\varepsilon$ where $\theta\sim (-\pi,\pi]$ and $\varepsilon$ is a zero-mean isotropic Gaussian noise $\mathcal{N}(0,0.05)$. A different point-wise Gaussian noise $\mathcal{N}(0,\omega_i)$ where $\omega_i\in(0.2,0.3)$ is further added to each point to simulate the random measurement corruption. In this experiment, we assume that $\omega_i$ is available during training. This situation is often encountered in many applications such as astronomy and computational biology. Both OCSVM and OCSMM are trained on the corrupted data. As opposed to the OCSVM that considers only the observed data points, the OCSMM also uses $\omega_i$ for every point via the kernel \eqref{eq:analytic-kernel}. Then, we consider a slightly more complicate data generated by $x = r\cdot\cos(\theta)$ and $y = r\cdot\sin(\theta)$ where $r = \sin(4\theta)+2$ and $\theta \in (0,2\pi]$. The data used in both examples are illustrated in Figure \ref{fig:density}.

As illustrated by Figure \ref{fig:density}, the density function estimated by the OCSMM is relatively less susceptible to the additional corruption than that estimated by the OCSVM, and tends to estimate the true density more accurately. This is not surprising because we also take into account an additional information about the uncertainty. However, this experiment suggests that when dealing with uncertain data, it might be beneficial to also estimate the uncertainty, as commonly performed in astronomy, and incorporate it into the model. This scenario has not been fully investigated in AI and machine learning communities. Our framework provides one possible way to deal with such a scenario.

\begin{figure}[t!]
  \centering
  \includegraphics[width=\linewidth]{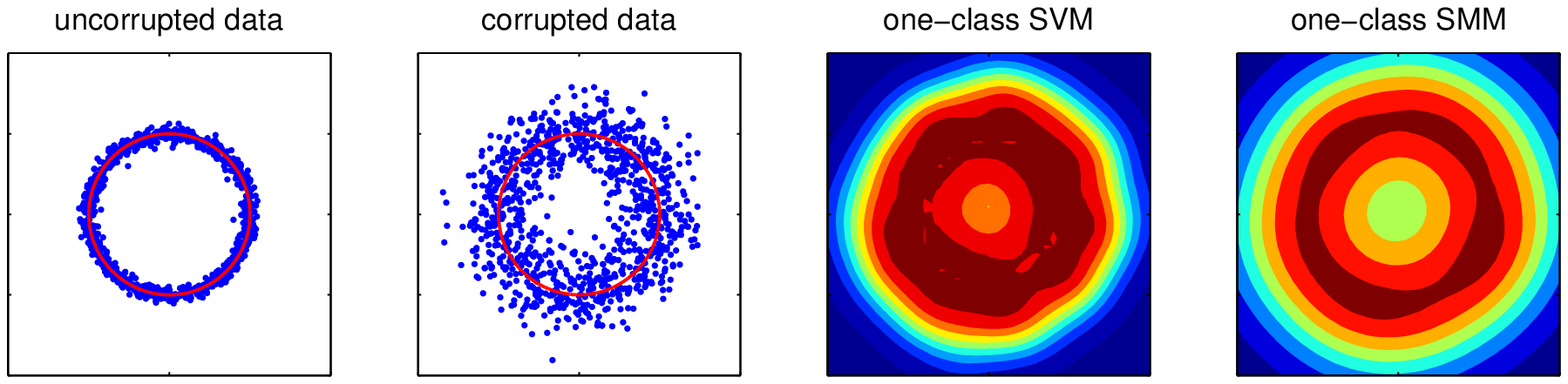} \\
  \includegraphics[width=\linewidth]{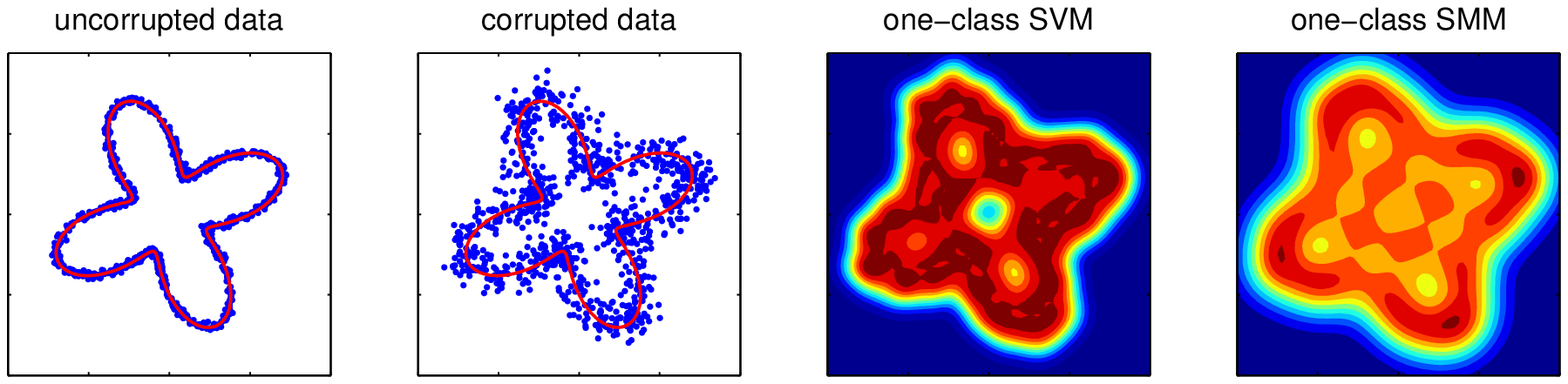}
  \caption{The density functions estimated by the OCSVM and the OCSMM using the corrupted data.} 
  \label{fig:density}
\end{figure}

\subsection{Sloan Digital Sky Survey}

Sloan Digital Sky Survey (SDSS)\footnote{See \href{http://www.sdss.org}{http://www.sdss.org} for the detail of the surveys.} consists of a series of massive spectroscopic surveys of the distant universe, the milky way galaxies, and extrasolar planetary systems. The SDSS datasets contain images and spectra of more than 930,000 galaxies and more than 120,000 quasars.

In this experiment, we are interested in identifying anomalous groups of galaxies, as previously studied in \citet{Poczos11:Divergence} and \citet{Xiong11:HPM,Liang11:FGM}. To replicate the experiments conducted in \citet{Xiong11:HPM}, we use the same dataset which consists of 505 spatial clusters of galaxies. Each of which contains about 10-15 galaxies. The data were preprocessed by PCA to reduce the 1000-dimensional features to 4-dimensional vectors. 

To evaluate the performance of different algorithms to detect group anomaly, we consider artificially random injections. Each anomalous group is constructed by randomly selecting galaxies. There are 50 anomalous groups of galaxies in total. Note that although these groups of galaxies contain usual galaxies, their aggregations are anomalous due to the way the groups are constructed. 

The average precision (AP) and area under the ROC curve (AUC) from 10 random repetitions are shown in Figure \ref{fig:auc-results}. Based on the average precision, KNN-L2, MGM, and OCSMM achieve similar results on this dataset and KNN-Renyi outperforms all other algorithms. On the other hand, the OCSMM and KNN-Renyi achieve highest AUC scores on this dataset. Moreover, it is clear that point anomaly detection using the OCSVM fails to detect group anomalies. 

\begin{figure}[t!]
  \centering
  \includegraphics[width=\linewidth]{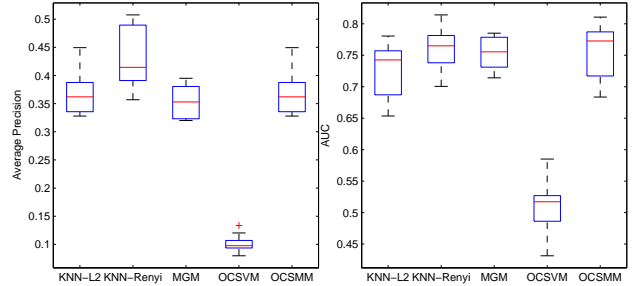}
  \caption{The average precision (AP) and area under the ROC curve (AUC) of different group anomaly detection algorithms on the SDSS dataset.}
  \label{fig:auc-results} 
\end{figure}


\subsection{High Energy Particle Physics}

\begin{figure*}[t] 
  \centering
  \includegraphics[width=\linewidth]{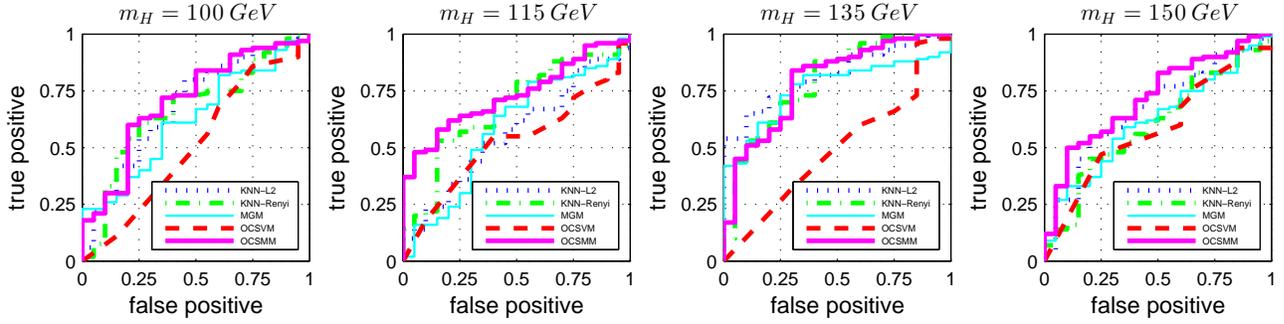} 
  \caption{The ROC of different group anomaly detection algorithms on the Higgs boson datasets with various Higgs masses $m_{H}$. The associated AUC scores for different settings, sorted in the same order appeared in the figure, are (0.6835,0.6655,0.6350,0.5125,\textbf{0.7085}), (0.5645,0.6783,0.5860,0.5263,\textbf{0.7305}), (\textbf{0.8190},0.7925,0.7630,0.4958,0.7950), and (0.6713,0.6027,0.6165,0.5862,\textbf{0.7200}).}
  \label{fig:higgs-roc} 
\end{figure*}

In this section, we demonstrate our group anomaly detection algorithm in high energy particle physics, which is largely the study of fundamental particles, e.g., neutrinos, and their interactions. Essentially, all particles and their dynamics can be described by a quantum field theory called the \emph{Standard Model}. Hence, given massive datasets from high-energy physics experiments, one is interested in discovering deviations from known Standard Model physics.

Searching for the Higgs boson, for example, has recently received much attention in particle physics and machine learning communities (see e.g., \citet{Bhat10:PP, Vatanan12:SSCA} and references therein). A new physical phenomena usually manifest themselves as tiny excesses of certain types of collision events among a vast background of known physics in particle detectors.

Anomalies occur as a cluster among the background data. The background data distribution contaminated by these anomalies will therefore be different from the true background distribution. It is very difficult to detect this difference in general because the contamination can be considerably small. In this experiment, we consider similar condition as in \citet{Vatanan12:SSCA} and generate data using the standard HEP Monte Carlo generators such as PYTHIA\footnote{\href{http://home.thep.lu.se/~torbjorn/Pythia.html}{http://home.thep.lu.se/$\sim$torbjorn/Pythia.html}}. In particular, we consider a Monte Carlo simulated events where the Higgs is produced in association with the $W$ boson and decays into two bottom quarks.

The data vector consists of 5 variables $(p_x,p_y,p_z,e,m)$ corresponding to different characteristics of the topology of a collision event. The variables $p_x,p_y,p_z,e$ represents the momentum four-vector in units of GeV with $c=1$. The variable $m$ is the particle mass in the same unit. The signal looks slightly different for different Higgs masses $m_H$, which is an unknown free parameter in the Standard Model. In this experiment, we consider $m_H=100, 115, 135$, and $150$ GeV. We generate 120 groups of collision events, 100 of which contain only background signals, whereas the rest also contain the Higgs boson collision events. For each group, the number of observable particles ranges from 200 to 500 particles. The goal is to detect the anomalous groups of signals which might contain the Higgs boson without prior knowledge of $m_H$.

Figure \ref{fig:higgs-roc} depicts the ROC of different group anomaly detection algorithms. The OCSMM and KNN-based group anomaly detection algorithms tend to achieve competitive performance and outperform the MGM algorithm. Moreover, it is clear that traditional point anomaly detection algorithm fails to detect high-level anomalous structures.

\section{Conclusions and Discussions} 
\label{sec:conclusions}

To conclude, we propose a simple and efficient algorithm for detecting group anomalies called one-class support measure machine (OCSMM). To handle aggregate behaviors of data points, groups are represented as probability distributions which account for higher-order information arising from those behaviors. The set of distributions are represented as mean functions in the RKHS via the kernel mean embedding. We also extend the relationship between the OCSVM and the KDE to the OCSMM in the context of variable kernel density estimation, bridging the gap between large-margin approach and kernel density estimation. We demonstrate the proposed algorithm on both synthetic and real-world datasets, which achieve competitive results compared to existing group anomaly detection techniques.

It is vital to note the differences between the OCSMM and hierarchical probabilistic models such as MGM and FGM. Firstly, the probabilistic models assume that data are generated according to some parametric distributions, i.e., mixture of Gaussian, whereas the OCSMM is nonparametric in the sense that no assumption is made about the distributions. It is therefore applicable to a wider range of applications. Secondly, the probabilistic models follow a bottom-up approach. That is, detecting group-based anomalies requires point-based anomaly detection. Thus, the performance also depends on how well anomalous points can be detected. Furthermore, it is computational expensive and may not be suitable for large-scale datasets. On the other hand, the OCSMM adopts the top-down approach by detecting the group-based anomalies directly. If one is interested in finding anomalous points, this can be done subsequently in a group-wise manner. As a result, the top-down approach is generally less computational expensive and can be used efficiently for online applications and large-scale datasets.

 
{\small
\bibliographystyle{abbrvnat}
\bibliography{ocsmm}}

\end{document}